\documentclass[lettersize,journal]{IEEEtran}
\usepackage{amsmath,amsfonts}
\usepackage{array}
\usepackage{textcomp}
\usepackage{stfloats}
\usepackage{url}
\usepackage{verbatim}
\usepackage{graphicx}
\usepackage{hashem}
\usepackage{booktabs}
\usepackage{float}

\begin{document}

\title{Gradient Scarcity with Bilevel Optimization\\ for Graph Learning}

\author{Hashem Ghanem, Samuel Vaiter, and Nicolas Keriven
\thanks{The authors acknowledge the support of ANR Grava ANR-18-CE40-0005 and ANR GRandMa ANR-21-CE23-0006.}%
\thanks{HG is with CNRS, IMB, Univ. de Bourgogne; NK is with CNRS, IRISA, Univ. Rennes 1; SV is with CNRS, LJAD, Univ. Cote d’Azur.}}

\markboth{}%
{Gradient scarcity with Bilevel Optimization for Graph Learning}

\maketitle

\begin{abstract}
  A common issue in graph learning under the semi-supervised setting  is referred to as \emph{gradient scarcity}.
  That is, learning graphs by minimizing a loss on a subset of nodes causes edges between unlabelled nodes that are far from labelled ones to receive zero gradients.
  The phenomenon was first described when optimizing the graph and the weights of a Graph Neural Network (\GCN) with a joint optimization algorithm.
  In this work, we give a precise mathematical characterization of this phenomenon, and prove that it also emerges in \emph{bilevel} optimization, where additional dependency exists between the parameters of the problem.
  While for \GCNs gradient scarcity occurs due to their finite receptive field, we show that it also occurs with the Laplacian regularization model, in the sense that gradients amplitude decreases exponentially with distance to labelled nodes.
  To alleviate this issue, we study several solutions: we propose to resort to latent graph learning using a Graph-to-Graph model (\GtoG), graph regularization to impose a prior structure on the graph, or optimizing  on a larger graph than the original one with a reduced diameter.
   Our experiments on synthetic and real datasets validate our analysis and prove the efficiency of the proposed solutions.
\end{abstract}


\section{Introduction}

\IEEEPARstart{T}{he}  expensive cost of labelling data represents a challenge as the amount of generated data has been growing exponentially.
As a result, it is common to observe both labelled and unlabelled data points, the latter being usually the vast majority.
Learning tasks on datasets which comprise both labeled and unlabeled points is referred to as Semi-Supervised Learning (SSL).
SSL is usually handled with extra assumptions on the data. The main one, called \emph{homophily}, refers to the fact that ``nearby'' points are likely to have similar labels \cite{wang2006label}.
Moreover, points in many applications represent entities that are naturally linked to each other, \eg in biology \cite{liu2018constrained} or social media \cite{liben2003link}. There again, linked entities are likely to share the same label, which underlines the importance of exploiting the links when solving SSL inference problems.
Consequently, various graph-based methods have been developed for SSL.

One issue with such methods is that their performance is highly dependent on the graph quality.
This issue poses a significant challenge as real-world graphs are inherently noisy, significantly degrading the performance and leading to sub-optimal solutions.
Many \emph{graph learning} algorithms have thus been proposed in the literature to overcome this issue. Among these methods, a mainstream approach is to optimize the graph structure by means of optimizing the performance directly in the downstream task.

This approach involves generating a graph that, when used by a graph-based model, minimizes some loss on labelled nodes. However, graph-based models themselves require an optimization process to minimize the classification loss.
Therefore, both the graph learning process and the graph-based classification model need to learn by minimizing the ``same'' loss. There are three common gradient-based optimization routines that can be applied for this purpose. In the first routine, both the graph and the graph-based model are \emph{jointly} optimized. In the second, \emph{alternate minimization}, one is fixed while the other is updated in one iteration, and vice versa in the next iteration. The third routine is \emph{bilevel optimization}, that is an (outer) graph learning optimization problem involving the optimal model obtained by an (inner) optimization.

For Graph Neural Network (\GCN), the authors in~\cite{fatemi2021slaps} show that joint optimization leads to \emph{gradient scarcity}.
It refers to the fact that connections between unlabelled nodes ``far'' from the labelled ones receive \emph{zero gradients}, \ie they receive no supervision during the optimization and are not learned. This is due to the finite receptive field (depth) of message-passing \GCN.
In this work, we focus on bilevel optimization and prove that gradient scarcity also occurs for all \GCNs, despite additional dependency between the parameters. We also prove that this issue emerges with other graph-based classifiers, including Laplacian-based labels propagation, which, unlike \GCNs, has infinite receptive field.

\subsection{Semi-Supervised Learning}

A graph $\calG$ is a pair $(V, E)$, where $V$ is a set of $n$ nodes and $E\subseteq V\times V$ is a set of edges.
We represent a graph by its adjacency matrix $\V{A} \in \bbR ^{n\times n}$, where $\V{A}_{i,j}$ is the weight of the edge between nodes $i,j$. We denote by $\V{X} \in \bbR^{n\times p}$ the feature matrix whose rows include the features of corresponding nodes, and by $\V{Y}\in\bbR^{n}$ the vector of node labels.

We look at transductive SSL problems, where we have a set of points, a subset of which is labelled, and the goal is to approximate the labelling function on unlabelled points.
Formally, we have $(\V{X}_{obs},~\calG_{obs},~\V{Y}_{obs})$, where $\calG_{obs}$ is the observed graph, $\V{X}_{obs}$ are the observed node features (we will drop the subscript and write $\V{X}$ in the rest of the paper) and $\V{Y}_{obs}\in\bbR^{n}$ contains the labels of a subset of points at coordinates $i \in V_{tr}\subset V$ and, \eg not-a-number ``\emph{NaN}'' outside of $V_{tr}$.
There are roughly two main strategies to solve SSL problems. The first is to \emph{propagate} known labels using a \emph{regularization} process. Predicted labels reads the following:
\begin{equation}
    \V{Y}_{\Reg}(\V{A})
    \!\in\! \argmin_{\V{Y}}
     \tfrac{1}{|V_{tr}|}\!\sum_{i\in V_{tr}}\!\ell({\V{Y}}_i, (\V{Y}_{obs})_i)\!+\!
    \lambda R(\V{Y}\!,\!\V{A}),
    \label{eq:inner-problem-laplace}
\end{equation}
where 
$\ell$ is a smooth loss function commonly chosen to be the Categorical Cross Entropy (CCE) loss for classification, and the Mean Square Error (MSE) for regression, $R$ is a regularization function, and $\lambda$ is a balancing parameter. A popular choice is the \emph{Laplacian regularization}:
\begin{equation}
    \label{eq:lapl_reg}
    R(\V{Y}, \V{A}) = \frac{1}{|E|}\sum_{i,j}\V{A}_{ij}(\V{Y}_i - \V{Y}_j)^2 = \frac{1}{|E|}\V{Y}^\top \V{L} \V{Y}\enspace,
\end{equation}
where $\V{L} = \V{L}(\V{A}) = \V{D} - \V{A}$ is the Laplacian of the graph. Note that here the node features $\V{X}$ are not used.

The second main strategy for SSL is to train a \emph{parametric model} ${\V{Y}}_W(\V{X}, \V{A})$ parameterized by the weights $W$, such as \GCNs. The objective reads:
\begin{align}
&\V{Y}_{\GCN}(\V{A}) = {\V{Y}}_{W^\star}(\V{X}, \V{A}), \text{ where} \notag \\
& W^\star = \argmin_W \frac{1}{|V_{tr}|}\sum_{i\in V_{tr}}  \ell\Big(\big({\V{Y}}_W(\V{X}, \V{A})\big)_i, (\V{Y}_{obs})_i\Big)\enspace. \label{eq:inner-problem-gcn}
\end{align}
In this paper, we use message-passing \GCNs with sum aggregation. The first layer is $\V{X}^{[0]} = \V{X}$, propagated as
\begin{equation}\label{eq:gcn_layer}
    \V{X}^{[l]} = \phi(\V{X}^{[l-1]} \V{W}_1^{[l]} + \V{A}\V{X}^{[l-1]} \V{W}_2^{[l]} + \V{1}_n (\V{b}^{[l]})^\top) \enspace,
\end{equation}
where $\V{W}_1^{[l]}, \V{W}_2^{[l]} \in \mathbb{R}^{d_{l-1} \times d_{l}}$ are learnable weights, $\V{b}^{[l]} \in \mathbb{R}^{d_{l}}$ is a learnable bias, $d_l$ is the output dimensionality of the $l$-th layer, $\V{1}_n = [1, \ldots, 1]^\top\in \bbR^n$, and $\phi$ is a non-linear function applied element-wise. The output ${\V{Y}}_{W}(\V{X}, \V{A}) = \V{X}^{[k]}$ is obtained after $k$ rounds of message passing, and the parameters are gathered as $W = \{ \V{W}_1^{[l]}, \V{W}_2^{[l]}, \V{b}^{[l]}\}_{l=1}^k$.

\subsection{Bilevel optimization for graph learning}
We consider the case where the graph objective function is a function of the trained classifier, that is, we look at a \emph{bilevel optimization}. Using a second set of labelled nodes $V_{out} \subset V$ distinct from $V_{tr}$  and given a set of admissible adjacency matrices $\mathcal{A}$, the bilevel optimization is cast as
\begin{equation}\label{eq:bilevel_problem_learn_A}
\widehat{\V{A}}\!\in\!\argmin_{\V{A}\in \mathcal{A}} F_{out}(\V{A})\!=\!\tfrac{1}{|V_{out}|} \sum_{i\in V_{out}}  \ell(\V{Y}(\V{A})_i, (\V{Y}_{obs})_i),
\end{equation}
such that
$\V{Y}(\V{A}) = \V{Y}_{\GCN}(\V{A})$ or $\V{Y}(\V{A}) 
 = \V{Y}_{\Reg}(\V{A})$.
That is, the minimization of the objective function $F_{out}$, called the \emph{outer} optimization problem, involves $\V{Y}(\V{A})$, which is itself the result of an \emph{inner} optimization problem, either \eqref{eq:inner-problem-laplace} over $\V{Y}$ or \eqref{eq:inner-problem-gcn} over $W$.
Several models are possible for $\mathcal{A}$:

\textbf{Full learning}: $\mathcal{A}=[a,b]^{n \times n}$ is the set of all weighted adjacency matrices (generally with some bounds $a,b$ on the weights). This choice necessarily leads to an impractical quadratic complexity on the minimization.

\textbf{Edge refinement}: the learned adjacency matrix has the same zero-pattern as the observed adjacency matrix, that is, we learn weights only on existing edges.
\[
    \mathcal{A} = \{
    \V{A} \in [a,b]^{n \times n} | \V{A}_{ij} = 0 \text{ when } (\V{A}_{obs})_{ij} = 0 
    \} .
\]
The complexity is proportional to the number of edges, generally less than quadratic in $n$ as graphs tend to be sparse.

\textbf{Generalized edge refinement}: same principle, but the zero-pattern is given by a modification of the observed adjacency matrix. For instance, taking the zero-pattern of $\V{A}_{obs}^r$ yields an edge between neighbors that are less than $r$-hop from each other in $\calG_{obs}$, where nodes $i$ and $j$ are $r$-hop from each other if the length of the shortest path between them in $\calG_{obs}$ is $r$.

\textbf{Latent graph learning}: the learned graph is the output of a parametric model, that takes as input the node features and the observed graph: $\mathcal{A} = \{\V{A} = f_\theta(\V{A}_{obs}, \V{X})\}$. We will refer to such models as \emph{Graph-to-Graph} (\GtoG).

Both the inner and the problems are treated by a gradient-based algorithms. We refer to the outer gradient $\nabla F_{out}$, whether with respect to $\V{A}$ or $\theta$, as \emph{hypergradient}.

\subsection{Contributions}
Previous works observed gradient scarcity when learning the graph and a \GCN classifier with joint optimization.
Indeed, a $k$-layer \GCN computes the label of a node using only information from $r$-hops far nodes with $r\leq k$.
This label is then not a function of edges connecting nodes outside of this neighborhood, and the term in the classification loss corresponding to this label returns null gradients on those distant edges.
However, it is not straightforward to extend this argument for bilevel optimization.
Specifically, the previous discussion assumes that the trained weights of the \GCN after gradient-based do not depend on the adjacency matrix $\V{A}$, which is not the case in the bilevel setting.
Moreover, if the problem holds in this setting, the roles of $V_{tr}$ and $V_{out}$ need to be clarified. 
Another question is if this problem is mitigated by resorting to graph-based models with infinite receptive field, \eg the Laplacian regularization.

In this work, \textbf{we prove that hypergradient scarcity occurs under the bilevel optimization setting when adopting \GCNs as a classifier.} We show that using a $k$-layer \GCN induces null hypergradients on edges between nodes $k$-hop from labelled nodes in $V_{tr}\cup V_{out}$.
\textbf{For the Laplacian regularization, we prove that the problem persists}, as hypergradients are exponentially damped with distance from labelled nodes.
\textbf{We empirically validate our findings}. Then, we test three possible strategies to solve this issue: latent graph learning with \GtoG models, graph regularization and refining a power of the given adjacency matrix. Furthermore, we empirically \textbf{distinguish between hypergradient scarcity and overfitting}, in the sense that solving the former does not necessarily resolve the latter. 
To the best of our knowledge, this is the first work that mathematically tackles the gradient scarcity problem for bilevel optimization of graphs, and examines the phenomenon for models with infinite receptive field.

\section{Related work}\label{sec:related_work_bias}
Bilevel optimization is used in many applications like multi-task and meta learning \cite{bennett2006model,flamary2014learning,franceschi2018bilevel}. See \cite{colson2007overview} for a review of applications in different fields.

Graph learning gained in importance since real-world graphs usually have corrupted edges.
The first way used to construct these graphs might be the $k$-nearest neighbors technique \cite{roweis2000nonlinear,tenenbaum2000global} and its variants, but with shortcomings: we have to choose the number of nearest neighbors to consider and the associated similarity criterion. Here, we consider situations where the graph learning problem are formulated as a \emph{supervised} bilevel optimization problem.
In \cite{franceschi2019learning}, authors learn the parameters of Bernoulli probability distributions over independent random edges.
The problem is similarly framed as a bilevel optimization, where these parameters are optimized to minimize the \GCN's validation loss.
Similar to \cref{eq:bilevel_problem_learn_A} with full learning  of $\V{A}$, this method includes learning $n^2$ parameters which limits scalability.
In \cite{stretcu2019graph}, a state-of-the-art method referred to as Graph Agreement Model ($GAM$) is proposed to learn graphs for SSL problems by penalizing the absence of an edge between nodes with the same label. Thereby the penalty is not explicitly a function of the used \GCN model, and the problem isn't bilevel.
In attention mechanisms, edge weights are re-evaluated after each \GCN layer based on similarity between node representations, \ie edge refinement. The similarity criterion is either user-defined like the dot product \cite{luong-etal-2015-effective,vaswani2017attention}, learned locally at each layer by a single-layer feed-forward network \cite{velickovic2018graph}, or a combination of both schemes \cite{kim2021how}.
In contrast to bilevel optimization, these mechanisms are trained with the \GCN model using joint optimization.
To alleviate overfitting resulted from learning the \GCN parameters and edge weights together, \cite{wang2020unifying} makes use of the Label Propagation model (LPA) \cite{zhu2005semi} to regularize the graph.
The proposed framework produces state-of-the-art results on node classification tasks.
However, authors adopt the joint optimization scenario.

Gradient scarcity was studied in \cite{fatemi2021slaps} where the authors looked at this problem with the intuition that learning a graph in SSL problems is done to improve performance in the downstream task, thus optimizing both requires such supervision that is not available in small labelled subsets.
Then, for downstream tasks adopting a $k$-layer \GCN classifier (with $k=2$ in their case), they identified what they refer to as the \emph{supervision starvation problem}, which states that edges between unlabelled nodes do not receive any supervision if they are at least $2$-hop from labelled nodes.
They quantify the starvation for the special case of Erdös-Rényi graphs.
Note that gradient scarcity and supervision starvation refer to the same phenomenon.

This issue cannot be resolved by adding more layers to the \GCN as this will increase its complexity on one hand, which means more data and labels are needed, and due to the oversmoothing issue on the other hand \cite{keriven2022not}.
To mitigate this issue and provide more supervision on the graph level, authors make use of the assumption that a good graph does not only perform well in labelling nodes, but also in denoising node features.
Therefore, they regularize the learned graph by a contrastive loss \cite{liu2021self,wu2021self,liu2022graph}, which evaluates its denoising performance, which overall results in a uni-level optimization.

That said, authors implicitly assumed no dependence between the \GCN weights and the graph when identifying gradient scarcity, which is the case in joint/alternating optimization schemes.
To the best of our knowledge, this issue has not yet been studied for the bilevel optimization setting.
Moreover, it is not clear if this problem is resolved with graph-based models with infinite receptive field, \eg the Laplacian regularization. We treat both these topics in our work.

In \cite{liu2022towards}, authors state that optimizing both the graph and a \GCN model under the supervision of a classification task introduces reliance on available labels, bias in the edge distribution and even reduce the span of potential application tasks.
Still, this statement is not accompanied with a theoretical justification, especially regarding the first two consequences.
To overcome this problem, authors suggested to avoid label-based graph optimization, and proposed an \emph{unsupervised} graph learning framework based on contrastive learning.
Although the unsupervised framework proved effective and competed state-of-the-art methods, we believe that labels contain informative knowledge that is not exploited when deploying unsupervised learners, and that better results are obtained by getting the best of both worlds.

\section{Hypergradient scarcity with \GCNs}\label{sec:gcn_bias}

In this section, we consider the bilevel optimization \eqref{eq:bilevel_problem_learn_A} in the \emph{edge refinement} setting, \ie we optimize the weight of every existing edge in $\V{A}_{obs}$, and the \GCN case $\V{Y}(\V{A}) = \V{Y}_{\GCN}(\V{A}) = \V{Y}_{W^\star}(\V{A}, \V{X})$.

In \cite{fatemi2021slaps}, the authors demonstrated that the predicted node label using a $2$-layer \GCN integrates information from nodes of distance less than two hops, \ie the label is not a function of edges connecting  nodes at least $2$-hop far away.
Consequently, when optimizing the graph by minimizing the classification error of that label via a gradient-based algorithm, these edges receive zero-valued gradients.
However, the authors used joint (or alternating) optimization of both the \GCN weights and the adjacency matrix, where the dependency between $W$ and $\V{A}$ is dropped, \ie $\V{J}_{W} (\V{A}) = \V{0}$.
\emph{This is not the case for bilevel optimization}.
In this section, we first examine the joint/alternating optimization schemes, and prove the existence of the problem for a generic number of layers $k$, similar to \cite{fatemi2021slaps}.
For the bilevel optimization setting, we then prove that the optimized weights $W^\star$ are not a function of edges connecting nodes at least $k$-hop from nodes in $V_{tr}$.
After that, we conclude that hypergradient scarcity holds in the bilevel setting for edges connecting nodes at least $k$-hop from nodes in the \emph{union} $V_{tr} \cup V_{out}$.

\subsection{Scarcity for joint or alternating optimization}

In this first result, we will assume that the weights $W$ do not depend on $\V{A}$, as is the case in joint/alternating minimization, and show gradient scarcity in $\V{Y}_W(\V{A}, \V{X})$.
This result uses the fact that $W$ does not depend on $\V{A}$, an hypothesis which is no longer satisfied in bilevel optimization.

\begin{theorem}\label{theorem:bias_gcns}
    Let $\V{Y}_W = \V{Y}_W(\V{A}, \V{X})$ be the output of a $k$-layer \GCN parameterized by $W$.
    Let $i,j,u$ be such that nodes $i,j$ are at least $k$-hop from node $u$.
    Assume that $\frac{\partial W}{\partial\V{A}_{i,j}} = \V{0}$. Then:
	  \begin{equation}\label{eq:zerograd_alternating}
            \frac{\partial (\V{Y}_W)_u}{\partial\V{A}_{i,j}} = 0  \enspace .
	  \end{equation}
\end{theorem}
\begin{proof}
    The proof is done by induction on $k$.
    For $k=1$, this is indeed the case since $\V{X}^{[0]} = \V{X}$ does not depend on $\V{A}$, and that $\V{A}_{i,j}$ does not belong to the row $\V{A}_{u,:}$ which is the only row in $\V{A}$ that contributes in the value $(\V{X}^{[1]})_{u,:}$.

    Assume that the statement is true for some arbitrary positive integer $k$, we show that it is also true for a $k+1$-layer \GCN.
    If $i,j$ are at least $(k+1)$-hop far from $u$, then clearly they are at least $k$-hop far from it too.
    Thus from the induction assumption, we have that $(\V{X}^{[k]})_{u,:}$ is independent of $\V{A}_{i,j}$.
    Also, $\V{W}_1^{[k+1]}$ does not depend on $\V{A}_{i,j}$ since we assume $\frac{\partial W}{\partial\V{A}_{i,j}} = \V{0}$. Therefore, $(\V{X}^{[k]} \V{W}_1^{[k+1]})_{u,:}$ in \eqref{eq:gcn_layer} does not depend on $\V{A}_{i,j}$ too.

    In a similar way, if $i,j$ are at least $(k+1)$-hop far from $u$, then they are at least $k$-hop far from any of its neighbors $v$ where $\V{A}_{u,v}\neq 0$.
    Therefore, if for all $v$, $\V{A}_{u,v}\neq 0$, then $\frac{\partial (\V{X}^{[k]})_{v,:}}{\partial\V{A}_{i,j}}= \V{0}$.
    Moreover, $\frac{\partial \V{W}_2^{[k+1]}}{\partial\V{A}_{i,j}} = \V{0}$ since we assume $\frac{\partial W}{\partial\V{A}_{i,j}} = \V{0}$.
    This makes $(\V{A}\V{X}^{[k]} \V{W}_2^{[k+1]})_{u,:} = \V{A}_{u,:}\V{X}^{[k]} \V{W}_2^{[k+1]}$ in \eqref{eq:gcn_layer} independent of $\V{A}_{i,j}$. This concludes the proof, as $ \frac{\partial (\V{Y}_W)_u}{\partial\V{A}_{i,j}} = \frac{\partial (\V{X}^{[k+1]})_u}{\partial\V{A}_{i,j}}= \V{0}$.
\end{proof}

\subsection{Gradient of the optimized weights}

\cref{theorem:bias_gcns} assumes that  $W$ is not a function of the edge
$\V{A}_{i,j}$, and states, in such case, that edges between nodes at least $k$-hop from the training nodes used to optimize the graph ($V_{out}$ in our case) receive no supervision.
However, in the bilevel optimization scenario, after the first outer iteration $W$ may depend on $\V{A}$.
The next theorem shows that gradient scarcity still occurs in the bilevel optimization framework, as the ``optimal'' weights used in practice are the result of a gradient-based algorithm. 
More precisely, we consider a sequence
    \begin{equation}\label{eq:GD}
        W_{t+1} = W_t - Q_t(W_t, \nabla_{W_t} F_{in}) \enspace,
    \end{equation}
where $Q_t$ is a smooth function. Note that $W_t$ does not necessarly converges towards the true optimal point $W^\star$
\begin{theorem}\label{theorem:W_dependent_of_far_edges_bilevel}
	Let $\V{A}$ be an input graph to a $k$-layer \GCN with weights $W$, and $W_t$ be the output obtained by optimizing \eqref{eq:inner-problem-gcn} for $W$ using a gradient-based iterates sequence.
    Let $i,j$ be nodes that are at least $k$-hop from any node in $V_{tr}$.
    Then, for all $t \in \mathbb{N}$,
	\begin{equation}\label{eq:zerograd-gnn-iterates}
		\frac{\partial W_t(\V{A})}{\partial\V{A}_{i,j}} = \V{0} \enspace.
	\end{equation}
\end{theorem}
\begin{proof}
    The proof is carried out by induction on the iteration index $t$ of the gradient-based optimizer. Denote by $F_{in}$ the objective function in \eqref{eq:inner-problem-gcn}.
    For $t = 0$, $W_0$ is the initialization of $W$ which is usually random and does not depend on $\V{A}$.
    For $t\geq 0$, we assume that $\frac{\partial W_t}{\partial\V{A}_{i,j}} = \V{0}$ and prove this must be true for $t+1$.
   By the chain rule, proving that $\frac{\partial (\nabla_{W_t} F_{in})}{\partial\V{A}_{i,j}} = \V{0}$ is sufficient to complete the proof.
    The gradient $\nabla_{W_t} F_{in}$ writes:
  \begin{equation*}
  \nabla_{W_t}F_{in}=\frac{1}{|V_{tr}|}\sum_{u\in V_{tr}}  \nabla_{W_t}\ell\Big(\big({\V{Y}}_{W_t}(\V{X}, \V{A})\big)_u, (\V{Y}_{obs})_u\Big)\enspace.
  \end{equation*}
   For all $u\in V_{tr}$, the term $\nabla_{W_t}\ell\Big(\big({\V{Y}}_{W_t}(\V{X}, \V{A})\big)_u, (\V{Y}_{obs})_u\Big)$ is a function of $W_t$ and $\big({\V{Y}}_{W_t}(\V{X}, \V{A})\big)_u$.
    But $\frac{\partial W_t}{\partial\V{A}_{i,j}} = \V{0}$ from the induction assumption, and, given that, we have $\frac{\partial \big({\V{Y}}_{W_t}(\V{X}, \V{A})\big)_u}{\partial\V{A}_{i,j}} = 0$ from \cref{theorem:bias_gcns}.
    Thus, we have for all $u \in V_{tr}$, $\frac{\partial}{\partial\V{A}_{i,j}} \nabla_{W_t}\ell\Big(\big({\V{Y}}_{W_t}(\V{X}, \V{A})\big)_u, (\V{Y}_{obs})_u\Big)= \V{0}$.
    This concludes the proof of~\eqref{eq:zerograd-gnn-iterates} as it gives $\frac{\partial (\nabla_{W_t} F_{in})}{\partial\V{A}_{i,j}} = \V{0}$.
\end{proof}

\subsection{Hypergradient scarcity}

Finally, we put the two previous results together. The next theorem states that within the bilevel optimization framework, edges between nodes at least $k$-hop from nodes in $V_{tr} \cup V_{out}$ receive no supervision.

\begin{theorem}\label{theorem:gcn_bias_bilevel}
	Let $\V{Y}_W$ be a $k$-layer \GCN. Assume that the inner optimization problem is solved with a gradient-based algorithm \eqref{eq:GD}.
    Then, for any pair of nodes $i,j$ at least $k$-hop from nodes in $V_{out}\cup V_{tr}$, we have $\frac{\partial F_{out}}{\partial\V{A}_{i,j}}= \V{0}$.
\end{theorem}
\begin{proof}
    Directly from \cref{theorem:W_dependent_of_far_edges_bilevel} we have that $\frac{\partial W_t(\V{A})}{\partial\V{A}_{i,j}} = \V{0}$ since $i,j$ are at least $k$-hop far from nodes in $V_{tr}$.
    This makes it possible to apply \cref{theorem:bias_gcns} to get that  $\forall u \in V_{out};~\frac{\partial ({\V{Y}}_{W_t})_u}{\partial\V{A}_{i,j}} = \V{0}$, as  $i,j$ are at least $k$-hop far from nodes in $V_{out}$ and $\frac{\partial W_t(\V{A})}{\partial\V{A}_{i,j}} = \V{0}$.
    This concludes the proof as $F_{out}$ penalizes the classification error only on nodes in $V_{out}$.
\end{proof}

Theorem~\ref{theorem:gcn_bias_bilevel} shows that the hypergradient scarcity problem emerges when solving edge refinement tasks: if two nodes are at least $k$-hop far from nodes in $V_{out}\cup V_{tr}$ in $\V{A}_{obs}$, the edge in between receives no hypergradients. In Section~\ref{sec:solution}, we will examine several strategies to mitigate this phenomenon.

\section{hypergradient scarcity with the Laplacian regularization}

In the previous section, we have seen how the finite receptive field of \GCNs directly induces the gradient scarcity problem. We now examine hypergradient scarcity when $\V{Y}(\V{A}) = \V{Y}_{\Reg}(\V{A})$ with the Laplacian regularization  \eqref{eq:lapl_reg}.
Indeed, in this case the inner problem \eqref{eq:inner-problem-laplace} does not have a finite receptive field, in the sense that in general $\frac{\partial \V{Y}(\V{A})}{\partial \V{A}_{ij}} \neq 0$ for all $i,j$, unlike the \GCN case as proven by \cref{theorem:bias_gcns}.

Surprisingly, we show that gradient scarcity still occurs in some sense.
More precisely, we prove that the magnitude of hypergradients decreases exponentially with the sum of the distance to $V_{tr}$ and the distance to $V_{out}$.

We consider the case where the downstream task is a regression problem, \ie $\ell$ in \cref{eq:inner-problem-laplace,eq:bilevel_problem_learn_A} is the MSE loss function.
Let $\V{S}_{in}\in\bbR^{n\times n}$ be the diagonal selection matrix whose entries equal $1$ if the corresponding node is in $V_{tr}$ and $0$ otherwise, the  solution $\V{Y}(\V{A})$ enjoys a closed-from expression:
\begin{align*}
	\V{Y}(\V{A}) = \left(\tilde{\V{S}}_{in} + \lambda \tilde{\V{L}} \right)^{-1}\tilde{\V{S}}_{in} \V{Y}_{obs}\enspace,
\end{align*}
where $\tilde{\V{S}}_{in} = \frac{\V{S}_{in}}{|V_{tr}|}$ and $\tilde{\V{L}} = \frac{\V{L}}{|E|}$.
For simplicity from now on, we denote $\V{B} = \tilde{\V{S}}_{in} + \lambda \tilde{\V{L}}$.
Then, we write $\V{Y}(\V{A})$ as:
\begin{equation}\label{eq:Laplacian_denoising_close_form_solution}
	\V{Y}(\V{A}) =\V{B}^{-1}\tilde{\V{S}}_{in}\V{Y}_{obs}\enspace.
\end{equation}
It is well-defined thanks to the following result.
\begin{lemma}\label{lemma:norm_strict_smaller_1}
    Assume that the graph is connected. The eigenvalues $\mu_i$ of $\V{B}$ satisfy, for all $i$:
    \begin{equation}
        0 < \mu_{\min} \leq \mu_i \leq \mu_{\max} \leq \frac{1}{|V_{tr}|} + 2\lambda\enspace .
    \end{equation}
\end{lemma}

Given that, we now state the main result of this section.
\begin{theorem}\label{theorem:laplacian}
    Let nodes $i,j$ be at least $k$-hop from $V_{out}$, and ${q}$-hop from $V_{tr}$. Then we have:
    \begin{equation}
        \left|\frac{\partial F_{out}}{\partial \V{A}_{ij}} \right| \lesssim \lambda\frac{\sqrt{|V_{out}|}+\mu_{\min}\sqrt{|V_{tr}|}|V_{out}|}{\mu_{\min}^3|V_{tr}||E|}  y_\infty^2 (1-\mu)^{q+k}\enspace,
    \end{equation}
    where $\mu = \frac{\mu_{\min}}{\mu_{\max}}$ and $y_\infty = \|\V{Y}_{obs}\|_\infty$.
\end{theorem}
Since both $\mu_{\min}, \mu_{\max}$ are strictly positive, as shown in the proof in \cref{sec:proof_laplacian}, then $0<1-\mu<1$.
Therefore, \cref{theorem:laplacian} states that the magnitude of the hypergradient is \emph{exponentially} damped in a speed that is at least proportional to $(1-u)^{q+k}$, leading to a form of hypergradient scarcity. 

The rest of this section is dedicated to proving Lemma \ref{lemma:norm_strict_smaller_1} and Thm.~\ref{theorem:laplacian}.
We first express $\V{Y}(\V{A})$ as a Neumann series, then we bound the derivative of terms in the resulted series, and by extension the gradient of $F_{out}$.

\subsection{Proof of Lemma \ref{lemma:norm_strict_smaller_1} and Neumann series expansion}
In the first step, we re-write the inverse of $\V{B}$ using Neumann series. We first need to prove that $\|\V{I}-\V{B}\| < 1$ (see \eg \cite{stewart1998matrix}), where $\V{I}\in\bbR^{n\times n}$ is the identity matrix. Remark that the eigenvalues of $\V{I} - \V{B}$ are $1-\mu_i$ where $\mu_1,\ldots, \mu_n$ are the eigenvalues of $\V{B}$. Assuming the graph is connected, the ordered eigenvalues $\{\nu_i\}_{i=0}^n$ of $\tilde{\V{L}}$ satisfy:
\begin{equation}
    0=\nu_1 < \nu_2 \leq \ldots \leq \nu_n \leq 2\enspace.
\end{equation}
The last inequality holds because $\|\V{L}\| \leq 2 d_{\max} \leq 2|E|$, where $d_{\max}$ is the maximum degree of the graph. Let  $\V{u}_1, \ldots, \V{u}_n$ be the eigenvectors of $\tilde{\V{L}}$, where $\V{u}_1 \propto \V{1}_n$ is associated to $0$.

\begin{proof}[Proof of Lemma \ref{lemma:norm_strict_smaller_1}]
    We have $\|\tilde{\V{S}}_{in}\| \leq 1/|V_{tr}|$ and $\|\tilde{\V{L}}\| \leq 2$ so by a triangular inequality the upper bound is proved.

    Using the eigendecomposition of $\tilde{\V{L}}$ and recalling that $\nu_1=0$, for any $\V{x} \in \bbR^n$:
    \begin{align*}
        \V{x}^\top \V{B} \V{x} &= \lambda\V{x}^\top \tilde{\V{L}} \V{x} + \V{x}^\top \tilde{\V{S}}_{in} \V{x} \\
        &= \lambda\sum_{i=2}^n (\V{x}^\top \V{u}_i)^2 \nu_i + \frac{\sum_{i \in V_{tr}} \V{x}_i^2}{|V_{tr}|}
    \end{align*}
    which, minimized over the unit sphere, gives the expression of $\mu_{\min}$. It is immediate that $\mu_{\min}\geq 0$. We prove that this value is strictly positive. Indeed, $\V{x}^\top \V{B} \V{x}=0$ implies that $\V{x}^\top \V{S}_{in} \V{x}=0$ and therefore $\V{x}_i=0$ for $i\in V_{tr}$, but also that $\V{L}\V{x}=0$ and therefore that $\V{x}\propto \V{1}_n$, which implies that $\V{x}=0$.
\end{proof}
Let $\tilde{\V{B}} = \V{B}/\mu_{\max}$, with eigenvalues
\begin{equation*}
	\quad 0 \leq 1-\mu_i/\mu_{\max} \leq 1-\mu < 1\enspace,
\end{equation*}
where $\mu = \frac{\mu_{\min}}{\mu_{\max}}$.
Using Neumann expansion, $\Bt ^{-1}$ writes:
\begin{equation}\label{eq:Neumann_series}
	\Bt ^{-1} = \sum_{r=0}^{\infty} (\V{I}-\Bt)^r
	\Rightarrow \V{Y}(\V{A}) = \sum_{r=0}^{\infty} (\V{I}-\Bt)^r \mu_{\max}^{-1}\tilde{\V{S}}_{in} \V{Y}_{obs}\enspace.
\end{equation}
We denote by $\V{T}_r$ the $r$-th term in $\V{Y}(\V{A})$:
\begin{equation}\label{eq:Tr}
	\V{T}_r = (\V{I}-\Bt)^r \mu_{\max}^{-1} \tilde{\V{S}}_{in} \V{Y}_{obs} \enspace.
\end{equation}
Note that since $\|\V{S}_{in}\V{Y}_{obs}\| \leq \sqrt{|V_{tr}|}\|\V{Y}_{obs}\|_\infty$, we have:
\begin{equation}
    \| \V{T}_r \| \leq \frac{\nu^r y_\infty }{\mu_{\max}\sqrt{|V_{tr}|}}\enspace,
\end{equation}
where $y_\infty = \|Y_{obs}\|_\infty$ and $\nu = 1-\mu$. Similarly, $\|\V{Y}({\V{A}})\| \leq \frac{ y_\infty}{\mu_{\min}\sqrt{|V_{tr}|}}$. Moreover, since $\V{I} - \Bt$ has the same zero-pattern than $\V{A}$ (except on the diagonal), if $u$ is more than $r$ hops from $V_{tr}$, we get $(\V{T}_r)_u = 0$.

\subsection{Gradient of $(\V{T}_r)_u$}
In the second step, we derive the formula of the gradient of $(\V{T}_r)_u$ \wrt $\V{A}$, and derive a bound on its magnitude as a function of $r$, $q$ the distance to $V_{tr}$, and $k$ the distance to $V_{out}$.
For $r>0$, the gradient of the $u$-th coefficient in $\V{T}_r$ \wrt $\V{I}-\Bt$ is:
\begin{align*}
	\nabla_{\V{I}-\Bt} (\V{T}_{r})_u =
				 \sum_{h=1}^{r} &\Big(\big((\V{I}-\Bt)^{r-h}\big)_{u,:}\Big)^\top\\
     &\times \big((\V{I}-\Bt)^{h-1} \mu_{\max}^{-1}\tilde{\V{S}}_{in} \V{Y}_{obs}\big)^\top ,
\end{align*}
by the product rule of differentiation, and we have
\begin{equation*}
	\nabla_{\V{I}-\Bt} (\V{T}_{r})_u = \sum_{h=1}^{r} \Big(\big((\V{I}-\Bt)^{r-h}\big)_{u,:}\Big)^\top(\V{T}_{h-1})^\top\enspace.
\end{equation*}
Using that $\V{I}-\Bt= \V{I}-\frac{1}{\mu_{\max}}(\tilde{\V{S}}_{in} + \lambda\tilde{\V{L}})$, we have
\begin{align}
	\nabla_{\tilde{\V{L}}} (\V{T}_{r})_u
	&=- \frac{\lambda}{\mu_{\max}} \nabla_{\V{I}-\Bt} (\V{T}_{r})_u \notag\\
	&= -\frac{\lambda}{\mu_{\max}} \sum_{h=1}^{r} \Big(\big((\V{I}-\Bt)^{r-h}\big)_{u,:}\Big)^\top(\V{T}_{h-1})^\top\enspace. \label{eq:Tr_gradient}
\end{align}
And finally, by deriving $\tilde{\V{L}}$ \wrt $\V{A}_{ij}$:
\begin{align}
	\frac{\partial (\V{T}_{r})_u}{\partial \V{A}_{ij}} = -\frac{\lambda}{|E|\mu_{\max}} \sum_{h=1}^{r} &\big((\V{I}-\Bt)^{r-h}\big)_{ui} (\V{T}_{h-1})_i \label{eq:partialVtr}\\
 &+ \big((\V{I}-\Bt)^{r-h}\big)_{uj} (\V{T}_{h-1})_j \notag \\
 &- \big((\V{I}-\Bt)^{r-h}\big)_{uj} (\V{T}_{h-1})_i \notag \\
 &- \big((\V{I}-\Bt)^{r-h}\big)_{ui} (\V{T}_{h-1})_j \notag\enspace,
 \end{align}
which allows us to prove the following.
\begin{lemma}\label{lem:partialVtr}
    Let $i,j,u$ such that: $i,j$ are at least $k$-hop from $u$, and at least ${q}$-hop from $V_{tr}$. Then:
    \begin{equation}
        \left|\frac{\partial (\V{T}_{r})_u}{\partial \V{A}_{ij}}\right| \leq \begin{cases}
        0 &\text{if ${q} + k>r$} \\
        \frac{4\lambda  y_\infty }{|E|\mu_{\max}^2\sqrt{|V_{tr}|}} (r-{q}-k) \nu^{r-1} &\text{otherwise.}
        \end{cases}
    \end{equation}
\end{lemma}
\begin{proof}
    Recall that $(\V{T}_r)_u=0$ if $u$ is more than $r$-hop from $V_{tr}$. Similarly, $((\V{I}-\Bt)^r)_{ui}=0$ if $u$ and $i$ are more than $r$-hop from each other. Hence, the term $\big((\V{I}-\Bt)^{r-h}\big)_{ui} (\V{T}_{h-1})_i$ appearing in \eqref{eq:partialVtr} is $0$ if $r-h<k$ or $h-1<{q}$, and bounded by $(\mu_{\max} \sqrt{|V_{tr}|})^{-1}\nu^{r-1} y_\infty$ otherwise. Similarly for the other terms, so the sum in \eqref{eq:partialVtr} runs over the indices $h$ that satisfy $q+1 \leq h \leq r-k$, which is either none if ${q+1} + k>r$, or $r-{q}-k$ terms otherwise, which concludes the proof.
\end{proof}

\subsection{Proof of \cref{theorem:laplacian}}\label{sec:proof_laplacian}

We finally examine the hypergradient, and prove an exponential damping rate of its magnitude with the cumulative distance to $V_{tr}$ and $V_{out}$ (the sum of both distances).
Considering $F_{out} = \|\V{S}_{out}(\V{Y}(\V{A}) - \V{Y}_{obs})\|^2$, where $\V{S}_{out}$ is the diagonal selection matrix whose diagonal entries equal $1$ if the corresponding node is in $V_{out}$ and $0$ otherwise, we have:
\begin{align*}
    \frac{\partial F_{out}}{\partial \V{A}_{ij}} &= 2 (\frac{\partial\V{Y}(\V{A})}{\partial\V{A}_{ij}})^\top \V{S}_{out}(\V{Y}(\V{A}) - \V{Y}_{obs}) \\
    &= 2\sum_{r=0}^\infty (\frac{\partial\V{T}_r}{\partial\V{A}_{ij}})^\top \V{S}_{out}(\V{Y}(\V{A}) - \V{Y}_{obs})\enspace.
\end{align*}
Using a triangular inequality, the bound on $\|\V{Y}(\V{A})\|$, and that $\|\V{S}_{out}\V{Y}_{obs}\| \leq \sqrt{|V_{out}|}y_\infty$ we get:
\begin{equation*}
    \|\V{S}_{out}(\V{Y}(\V{A}) - \V{Y}_{obs})\|\leq  \frac{1+\mu_{\min}\sqrt{|V_{tr}||V_{out}|}}{\mu_{\min}\sqrt{|V_{tr}|}}  y_\infty\enspace.
\end{equation*}
By incorporating the resulting inequality in bounding the hypergradient, and by noticing that $\V{S}_{out} = \V{S}_{out}^2$ we have:
\begin{align*}
    &\left|\frac{\partial F_{out}}{\partial \V{A}_{ij}}\right| \lesssim \frac{1+\mu_{\min}\sqrt{|V_{tr}||V_{out}|}}{\mu_{\min}\sqrt{|V_{tr}|}}  y_\infty \sum_{r=0}^\infty \|\V{S}_{out}\frac{\partial \V{T}_{r}}{\partial \V{A}_{ij}}\|\\
    &\lesssim \frac{1+\mu_{\min}\sqrt{|V_{tr}||V_{out}|}}{\mu_{\min}\sqrt{|V_{tr}|}}  y_\infty \sum_{r=0}^\infty \left(\sum_{u \in V_{out}} \left|\frac{\partial (\V{T}_r)_u}{\partial \V{A}_{ij}}\right|^2\right)^\frac12 \enspace.
\end{align*}
Using Lemma \ref{lem:partialVtr} and the hypotheses on $i$ and $j$, for $u$ in $V_{out}$, the term $\left|\frac{\partial (\V{T}_r)_u}{\partial \V{A}_{ij}}\right|$ is $0$ if $r < {q}+k+1$, and bounded by $\frac{4\lambda  y_\infty }{|E|\mu_{\max}^2\sqrt{|V_{tr}|}} (r-{q}-k) \nu^{r-1}$ otherwise. Hence:
\begin{align*}
    \left|\frac{\partial F_{out}}{\partial \V{A}_{ij}}\right|
    \lesssim &\lambda\frac{\sqrt{|V_{out}|}+\mu_{\min}\sqrt{|V_{tr}|}|V_{out}|}{\mu_{\min}|V_{tr}||E|\mu_{\max}^2}  y_\infty^2\\ &\times \sum_{r={q}+k+1}^\infty (r-{q}-k) \nu^{r-1} \enspace .
\end{align*}    
Then we see that for $\nu < 1$ we have
\begin{align*}
    \sum_{r={q}+k+1}^\infty (r-{q}-k) \nu^{r-1} &= \nu^{{q}+k}\sum_{r=1}^\infty r \nu^{r-1} \enspace,
\end{align*}
and $\sum_{r=1}^\infty r \nu^{r-1} = \frac{1}{(1-\nu)^2} = \frac{1}{\mu^2}$, which concludes the proof.

\section{Alleviating hypergradient scarcity}\label{sec:solution}
In this section, we review strategies to mitigate the hypergradient scarcity problem.
However, it is important that we make a distinction between resolving this issue and resolving the overfitting problem.
Indeed, if gradient scarcity is also caused by the limited quantity of available labelled data, it is important to avoid confusion with traditional overfitting.
In particular, while traditional overfitting is generally reduced by adding more training data, \emph{when optimizing edges far from labelled nodes gradient scarcity is observed regardless of the dataset size and the number of labels.}
We study several strategies to mitigate hypergradient scarcity in the bilevel setting, but we emphasize that they might not lead to a better generalization error altogether.

\textbf{Generalized edge refinement by optimizing $\V{A}_{obs}^r$.}
As hypergradient scarcity is observed on edges connecting nodes distant from the labelled ones, a natural fix is to reduce this distance.
One way to do that is by refining edges in a power of $\V{A}_{obs}$, as the matrix $\V{A}_{obs}^r$ includes $r$-edge long connections between nodes.
In our experiments we adopt $\V{A}_{obs}^6$ as this notably expands the graph but does not achieve the extreme case where the result is a complete graph.

\textbf{Graph regularization.}
Graph regularization is used to impose a prior structure on the learned graph, by adding a regularization term to $F_{out}$ to penalize graphs with undesirable properties.
For instance, \cite{kalofolias2016learn} proposes the regularization term
$-\gamma\V{1}_n^\top \log \V{A}\V{1}_n$ for some $\gamma>0$, to penalize low-degree nodes.
We use this choice in the experiments, but note that imposing task-related priors and regularization terms could lead to better performance. This will be the topic of future work.

\textbf{\GtoG for edge refinement.}
The third fix we suggest is latent graph learning using \GtoG models.
In the outer problem, we propose to replace optimizing edge weights by optimizing the parameters of a \GtoG model to predict similarity between nodes.
Let $\theta$ be the weights of this model, and  $\V{A}_\theta$ be its output graph, the \GtoG model we adopt is $(\V{A}_\theta)_{i,j} = \alpha\big(({\V{X}}_i-  {\V{X}}_j)^2\big)$,
where the square function is applied entrywise, $\alpha: \bbR^{p}  \to \bbR$ is a  Multi-Layer Perceptron (\MLP) model consisting of $k_{\GtoG}$ layers, each is of the form:
\begin{equation*}
  \V{X}^{[l]} = \phi^{[l]}(\V{X}^{[l-1]} \V{W}_1^{[l]} +  \V{1}_n (\V{b}^{[l]})^\top) \enspace,
\end{equation*}
where $\V{W}_1^{[l]} \in \mathbb{R}^{d_{l-1} \times d_{l}}, \V{b}^{[l]} \in \mathbb{R}^{d_{l}}$ are learnable parameters, and $d_l$ is the output dimensionality of the $l$-th layer. The parameters are gathered as $\theta = \{ \V{W}_1^{[l]}, \V{b}^{[l]}\}_{l=1}^{k_{\GtoG}}$.

\section{Experiments}

We\footnote{Our \textit{Python} implementation is available at \texttt{\url{https://github.com/hashemghanem/Gradients_scarcity_graph_learning}}.} use two synthetic datasets, the first one, called synthetic dataset 1, is designed to test the Laplacian regularization.
The second one is a binary classification dataset that can be used for both graph-based models. Due to the paradigm behind construction, we call it the cheaters dataset. We also illustrate our findings on the real-world Cora dataset.

\textbf{Bilevel optimization:}
The problem~\cref{eq:bilevel_problem_learn_A} is intractable as neither the solution of the
inner problem nor its gradient \wrt $\V{A}$ (or to $\theta$ with \GtoG models) has a closed form expression that can be evaluated.
To overcome this difficulty, we unroll~\cite{gregor2010learning} $\tau_{in}$ iterations of the gradient-based inner optimizer, then using the \texttt{Higher} package \cite{grefenstette2019generalized} to trace iterations and perform higher-order Automatic Differentiation (AD) to compute the hypergradient.
For both the inner and the outer optimizers, we consider the Adam algorithm~\cite{KingmaAdam}.

\textbf{Synthetic dataset 1:}
we sample \emph{i.i.d.} latent variables $\V{X}\in \bbR^{n\times p}$ for nodes uniformly at random from $[0, 1]$ with $n =1536, p=2$.
The ground-truth graph $\V{A}^\star$ is constructed \st $(\V{A}^\star)_{i,j} = 1$ if $\| \V{X}_i- \V{X}_j\|_2<\sigma$, and $0$ otherwise.
$\sigma$ is set to  $0.06$ in our experiments. 
Two distinct procedures were employed to sample the nodes that comprise $V_{tr}$, leading to two distinct realizations of the dataset as illustrated in \cref{fig:Laplace_scarcity_bilevel}(top).
The first procedure randomly samples $100$ nodes from the set $V$, hence $V_{tr}$ is well-spread, whereas
the second procedure selects the $100$ nodes with the smallest Euclidean distance to the point $(0.5, 0.5)$, thus $V_{tr}$ is concentrated in a small neighborhood in this case.
In both cases, we randomly sample $25$ nodes from $V$ to construct $V_{out}$.
The remaining nodes are equally divided between the validation and the test sets.
Then, each node $i$ in  $V_{tr}$ is labeled as follows:
\begin{equation*}
    ({\V{Y}_{obs}})_i = \zeta(e^{-\frac{(\V{X}_i-{\V{a}}_1)^2}{2(0.2)^2}}+
    e^{\frac{-(\V{X}_i-{\V{a}}_2)^2}{2(0.2)^2}}+
    e^{\frac{-(\V{X}_i-{\V{a}}_3)^2}{2(0.2)^2}})\enspace,
\end{equation*}
where ${\V{a}}_1, {\V{a}}_2, {\V{a}}_3$ are randomly sampled from $[0,1]^2$, and $\zeta$ is a scaling factor such that labels lie in $[0,1]$.
By this construction, the prior that the labelling function on the graph is smooth is met, and the Laplacian regularization can be applied as in \cref{eq:inner-problem-laplace}.

To generate labels for other nodes, we plug the labels of $V_{tr}$ and $\V{A}^\star$ in \cref{eq:inner-problem-laplace} with $\lambda = 1$, such that the solution holds the sought-for labels.
This way, the ground-truth graph actually plays a role in labelling nodes in $V_{out}$ and $V$.

The noisy observed graph is built upon random weights
\begin{equation*}
  (\V{A}_{obs})_{i,j} = \xi_{i,j} (\V{A}^\star)_{i,j}
  \quad \text{where} \quad
  \xi_{i,j} \sim \calU([0,1])  \enspace.
\end{equation*}
Experiments on this dataset are done with the Laplacian regularization in the inner problem  as in \cref{eq:inner-problem-laplace}.
\begin{figure}[t]
	\centering
	  \includegraphics[width=.24\textwidth]{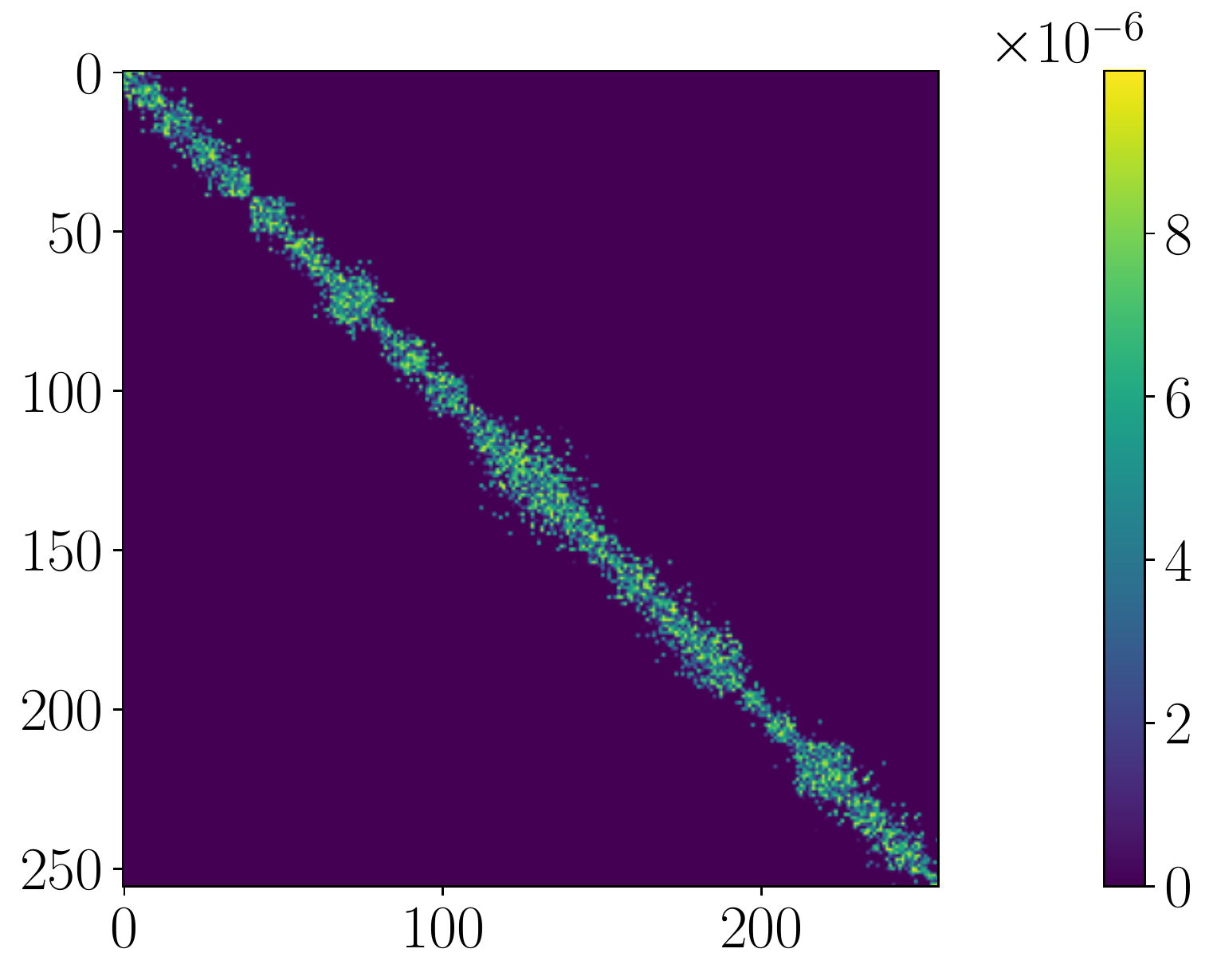}
	  \includegraphics[width=.24\textwidth]{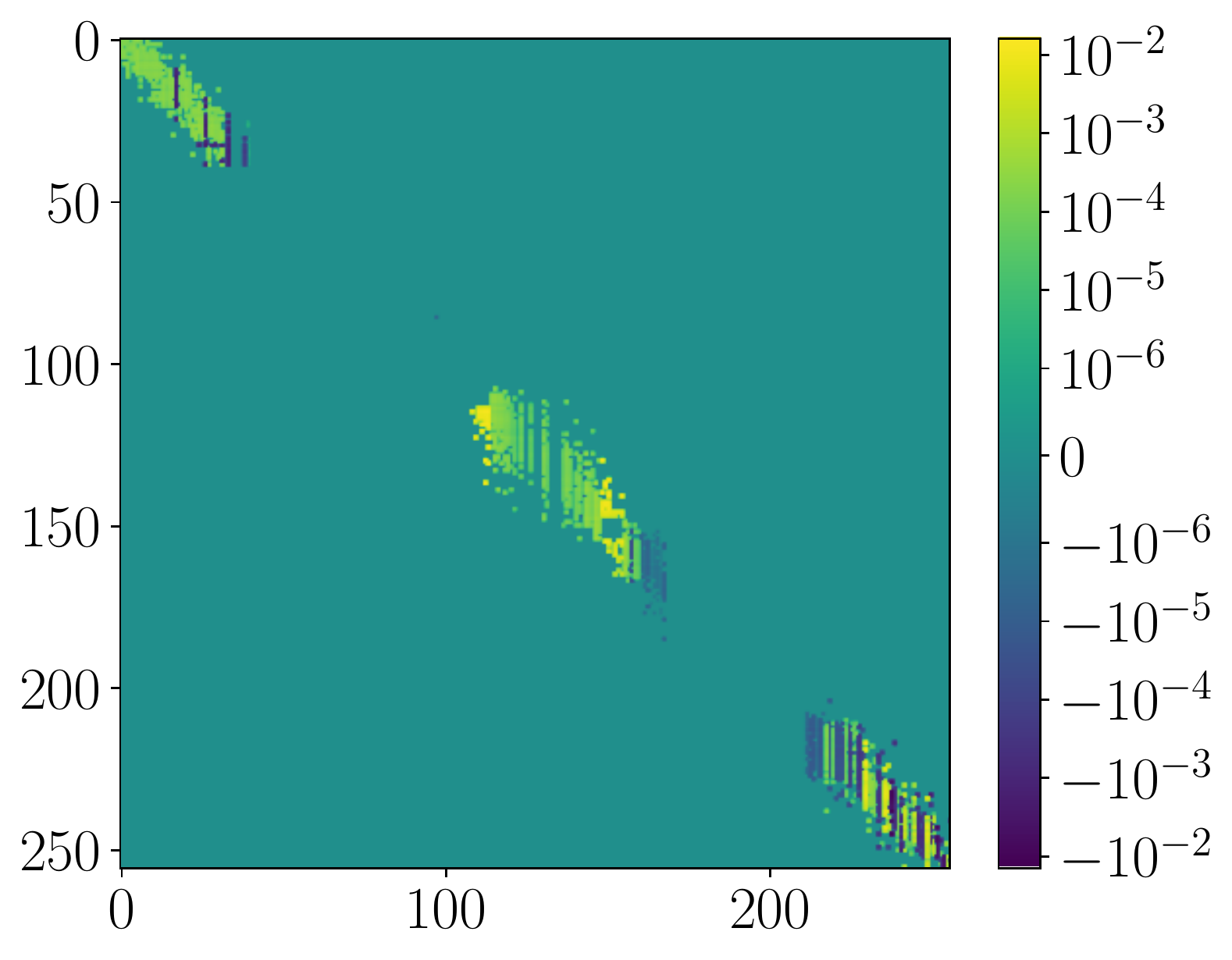}
	  \caption{Hypergradient scarcity observed when solving the edge refinement task with the bilevel optimization framework. We run the experiment on the cheaters dataset, and use a $2$-layer \GCN as a classifier.
	  Left: graph initialization. Right: hypergradient at an arbitrary outer iteration, namely $9$.  It is clear that the hypergradient on edges between unlabelled nodes far from the ones in $V_{out}\cup V_{tr}$ equals zero. Recall that $V_{tr} = \{0,1,\ldots, 32\} \cup  \{224,\ldots, 255\}$ and  $V_{out} = \{96,\dots, 160\}$.}
	  \label{fig:gcn_scarcity_bilevel}
\end{figure}

\textbf{Cheaters dataset:}
nodes in this graph represent students in an exam classroom.
Setting $n=256, p=10$, the \emph{i.i.d.} features $\V{X}\in \bbR^{256\times 10}$ are sampled uniformly at random from $[0, 1]$.
For a node $i$, $\V{X}_{i,0}$ represents the position of the according student in the classroom.
For visualization purposes we enumerate nodes following the ascending order of $\V{X}_{:,0}$.
The remaining $9$ features of a student represent the grades he is capable of scoring in the corresponding exam question.
However, students tend to cheat with their neighbors in the graph.
The ground-truth graph $\V{A}^\star$ is constructed as follows:
\begin{equation*}
    (\V{A}^\star)_{i,j} = \exp{(-\| \V{X}_{i,0}- \V{X}_{j,0}\|_2^2/2\sigma^2)} \enspace.
\end{equation*}
The observed graph $\V{A}_{obs}$ is drawn from a random  model as
\begin{equation*}
    (\V{A}_{obs})_{i,j}\sim  \mathop{\text{Ber}} \left( (\V{A}^\star)_{i,j}\right).
\end{equation*}
We set $\sigma = 0.027$ \st the number of edges in $\V{A}_{obs}$ approximates $n\log n$.
Students cheat such that their grades $\V{Y}_{grade}$ after the exam are
\begin{equation*}
    \V{Y}_{grade} =  \V{A}^\star \V{X}_{:,1:9} \V{1}_9\enspace.
\end{equation*}
A student passes the exam if his grade is greater than a threshold $\tau$, \ie $({\V{Y}_{obs}})_i=1$ if $(\V{Y}_{grade})_i>\tau$ and $0$  otherwise.
We put $\tau = 60$ so that approximately half of students pass the exam.
$V_{tr}$ includes nodes in  $\{0,1,\ldots, n/8\} \cup  \{7n/8,\ldots, n-1\}$, \ie near the two ends of the $1$-dimensional class. $V_{out} = \{3n/8,\dots, 5n/8\}$, \ie centered around the middle of the class. Remaining nodes are equally divided into a validation and a test set.
Experiments on this dataset are done with a \GCN classifier.

\textbf{Real-world dataset:}
we validate our findings on the Cora dataset \cite{Lu2003LinkbasedC}.
Cora is a citation datasets, where nodes represent research publications described by a bag of words, and edges stand for citations.
The task is to classify articles \wrt their topic.
In this work, we limit our experimentation on real-world datasets to the Cora dataset, as our empirical results are intended to establish a proof-of-concept. Therefore, we refrain from conducting experiments on other benchmark datasets.
\begin{figure*} 
	\centering
        \subfloat[]{   
	\includegraphics[width=.30\textwidth]{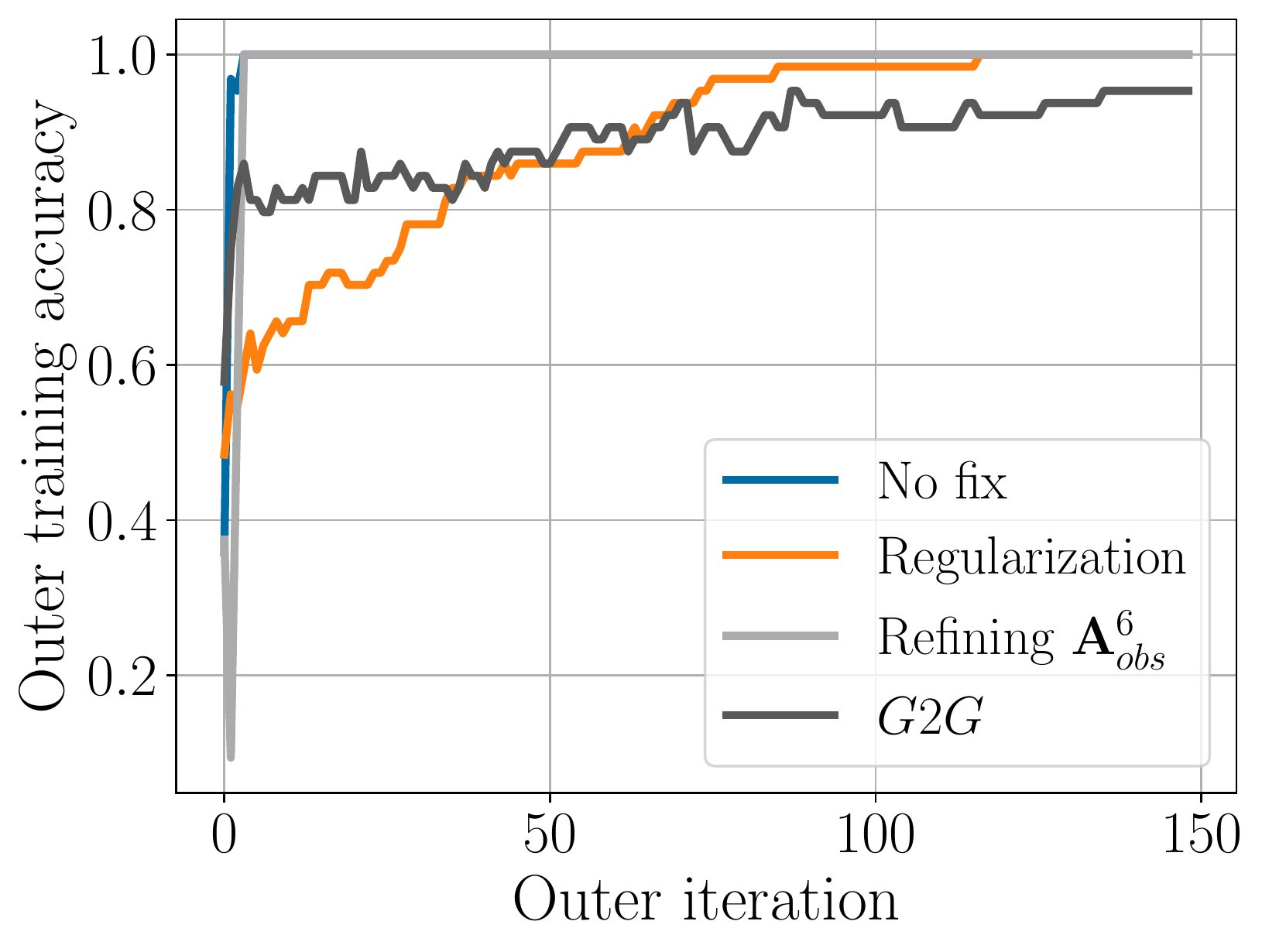}
	}
	\subfloat[]{
		\includegraphics[width=.30\textwidth]{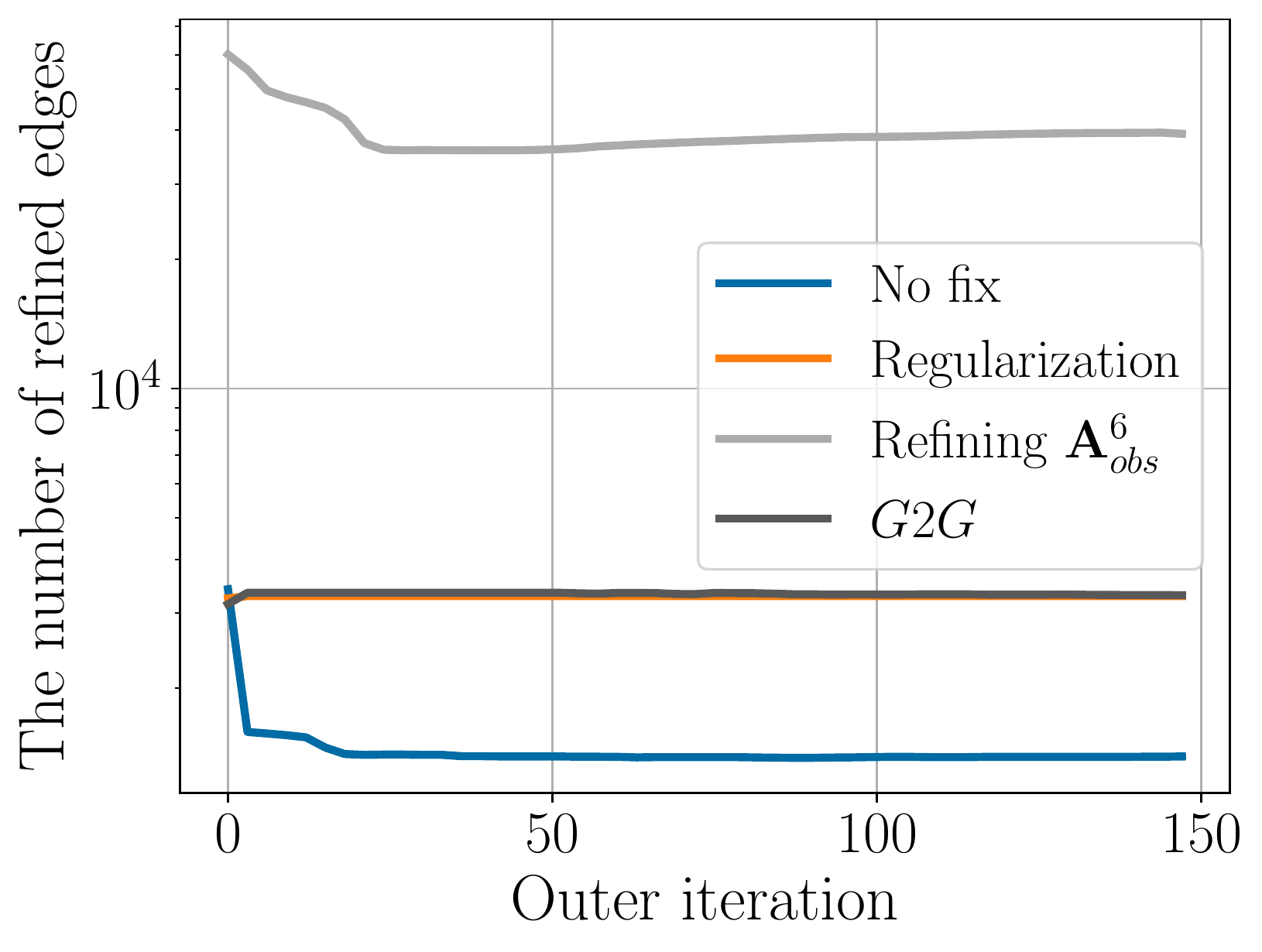}
	}
	\subfloat[]{
		\includegraphics[width=.30\textwidth]{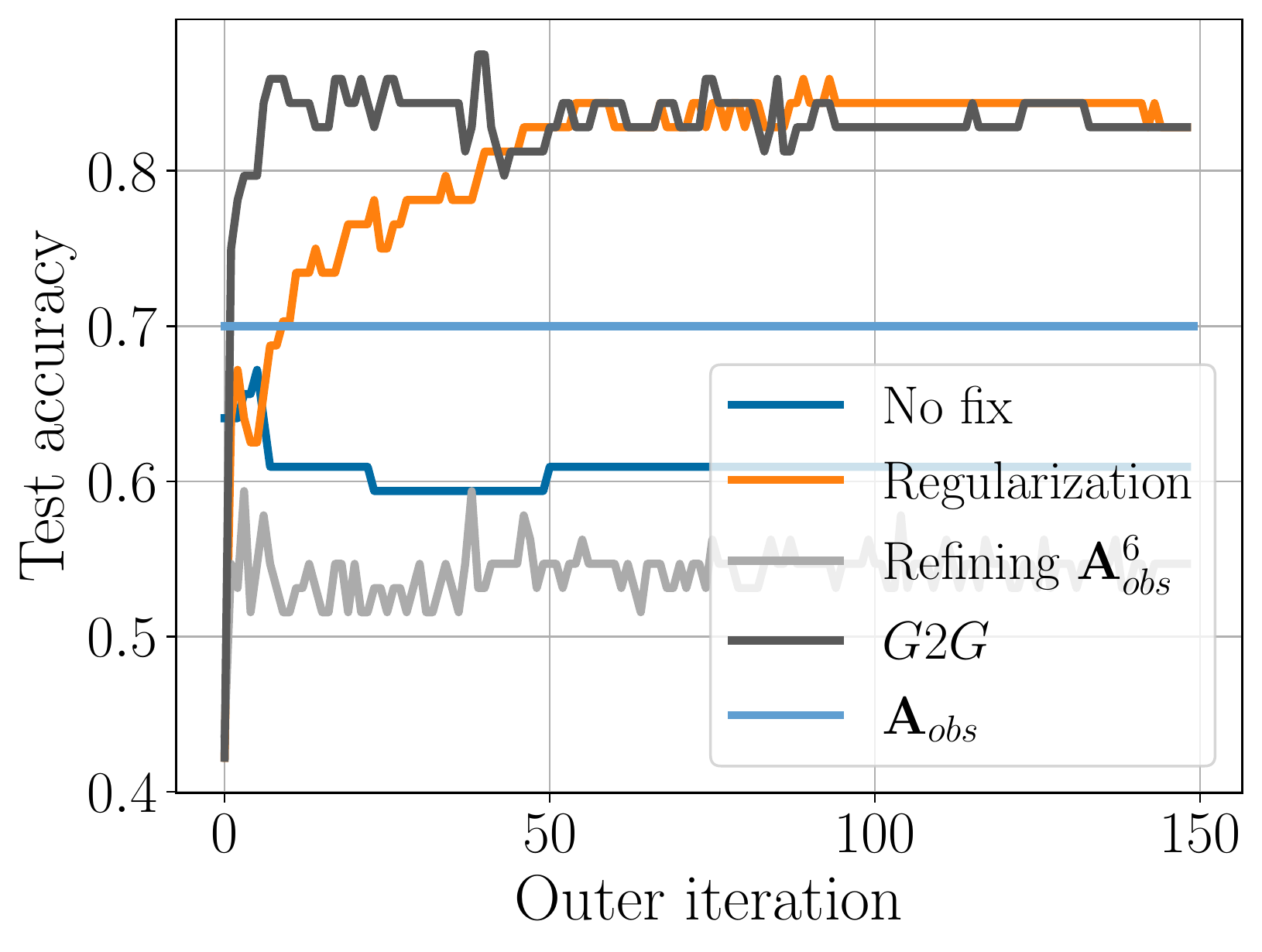}
	}
	\caption{Efficiency of proposed solutions to hypergradient scarcity \wrt the number of refined edges and the generalization capacity. An edge is considered well refined if its learned magnitude is larger than one percent of the maximum learned edge weight. The solutions are graph regularization with  $-\V{1}^\top \log \V{A}\V{1}$, latent graph learning using a \GtoG model, and generalized edge refinement by refining edges in $\V{A}_{obs}^6$. (a): training accuracy on $V_{out}$. (b): number of refined edges. (c) test accuracy.}
\label{fig:propsed_solutions}
\end{figure*}

{\textbf{Models:}}
 \GtoG and \GCN models are implemented using \textit{PyTorch} \cite{NEURIPS2019_9015} and \textit{PyTorch} Geometric \cite{Fey/Lenssen/2019}, respectively.
The function $\alpha$ in the \GtoG model is an \MLP with $2$ hidden layers, each is followed by the $ReLu$ activation function and has $16$ neurons for the cheaters dataset and $32$ neurons for Cora.
The \GCN has $1$ hidden layer of $8$ neurons for the cheaters dataset and $128$ for Cora. This layer is followed $ReLu$, while the output is followed by the  $softmax$ function.
\begin{figure}[t]
	\centering
   \includegraphics[width=.24\textwidth]{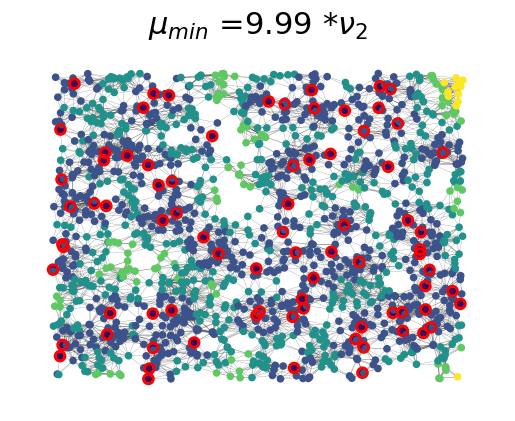}
   \includegraphics[width=.24\textwidth]{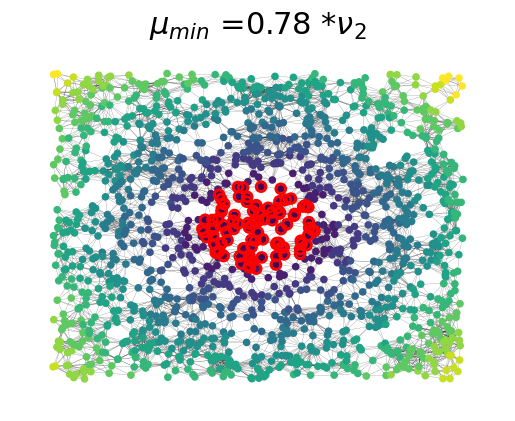}\\
   \includegraphics[width=.24\textwidth]{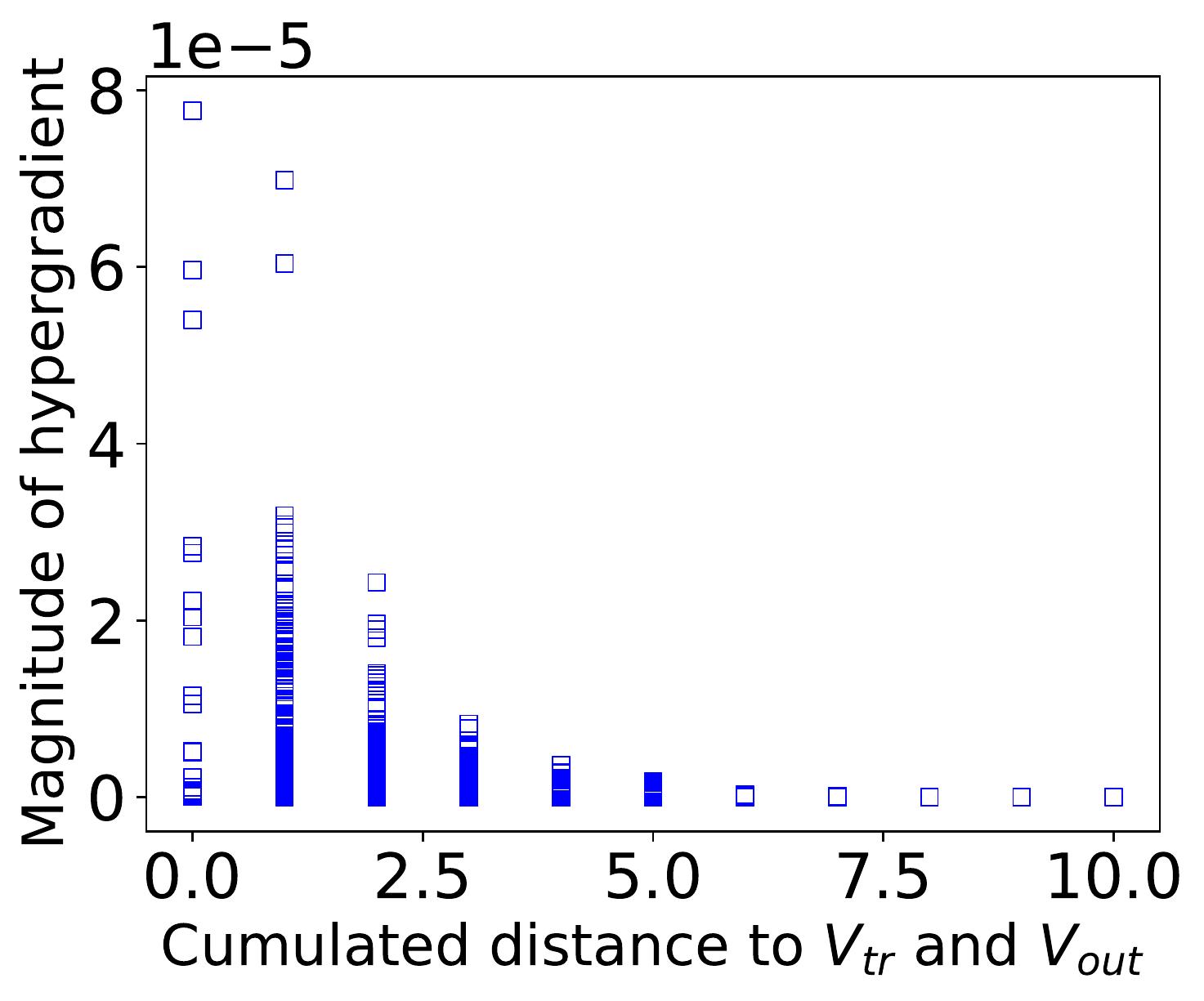}
   \includegraphics[width=.24\textwidth]{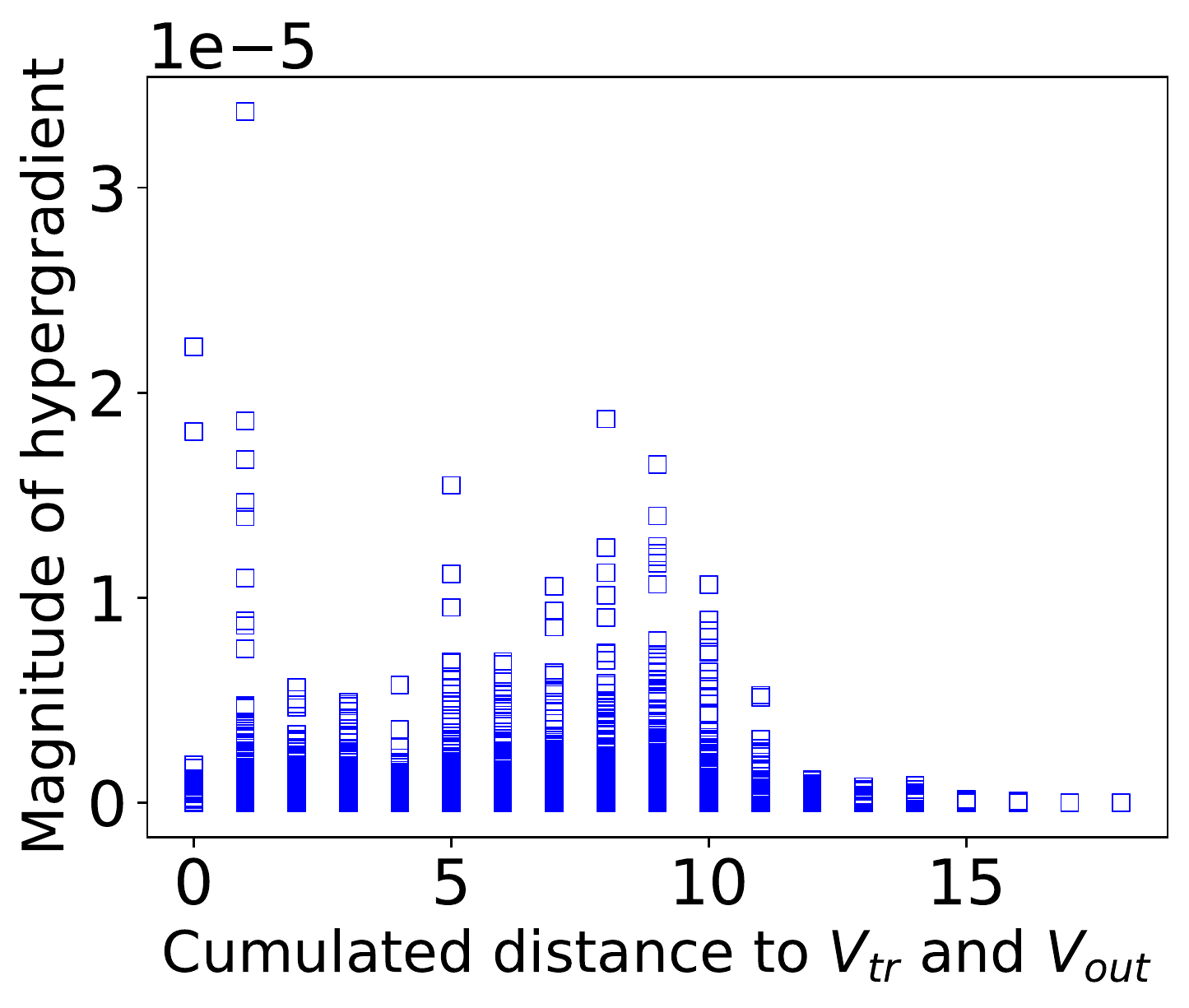}
	  \caption{hypergradient scarcity under the bilevel optimization setting on the synthetic dataset 1, adopting the Laplacian regularization in the inner problem.
   \textbf{Top}: illustration of the graph. The training nodes $V_{tr}$ are circled in red, the colors correspond to the distance to $V_{tr}$. The eigenvalue $\mu_{\min}$ is given as a ratio of the smallest positive eigenvalue of $\tilde{\V{L}}$.
   $V_{out}$ is randomly sampled from $V$ but not shown here.
   \textbf{Bottom}: Hypergradient magnitude $\left|\frac{\partial F_{out}}{\partial \V{A}_{ij}}\right|$ with respect to the \emph{sum} of distances to $V_{tr}$ and $V_{out}$.
   \textbf{Left}: the training set $V_{tr}$ is well-spread thereby aligned with the high-frequency eigenvectors of the graph, resulting in a \emph{high} $\mu_{\min}$. The decrease of the hypergradients is sharp with the distance.
   \textbf{Right}: $V_{tr}$ is aligned with the low-frequency eigenvectors of the graph, resulting in a \emph{low} $\mu_{\min}$. The decrease of hypergradients magnitude is not as sharp as the previous case.}
\label{fig:Laplace_scarcity_bilevel}
\end{figure}

{\textbf{Setup:}}
we use Adam as the inner and the outer optimizer with the default parameters of \textit{PyTorch}, except for the inner learning rate $\eta _{in}$ and the outer one $\eta _{out}$, which are tuned from the set $\{10^{-4}, 10^{-3}, \ldots, 10\}$. The best values were
$\eta_{in} =10^{-2}$ with \GCNs as a classifier, $\eta_{in} =10^{-1}$ and $\eta_{in} =10$ with the Laplacian regularization on Cora and on the synthetic dataset 1, respectively.
On the cheaters dataset, $\eta _{out} = 10^{-3}$ adopting a \GtoG model, while $\eta _{out} = 10^{-2}$ in other cases.
On the synthetic dataset 1, $\eta _{out} =10^{-1}$.
On Cora, $\eta _{out} = 10^{-2}$ in all experiments without a \GtoG model, Otherwise $\eta _{out} = 10^{-4}$ adopting the \GCN classifier, and $\eta _{out} = 10^{-3}$ adopting the Laplacian regularization.
We set, with a grid search, $\tau_{in}$ to $200$ for the cheaters dataset, $500$ for the synthetic dataset 1 and Cora adopting the Laplacian regularization, and $100$ for Cora with a \GCN classifier.
In experiments on the cheaters dataset, we multiply the default initialization of the last layer of the  \GtoG  model by $10^{-5}$ \st its output edges at the first iteration are of small magnitude. We adopt this strategy to measure the level of scarcity by counting the number of learned edges of magnitude greater than a chosen threshold.
\GCN weights $W$ and the initialization of labels when using the Laplacian regularization are initialized at random after each outer iteration, using Xavier initialization and uniformly at random from $[0,1]$, respectively.
Edges to be refined are initialized uniformly at random from $[0,1]$, except for experiments on the cheaters dataset where the interval becomes $10^{-5}*[0,1]$.
We set the number of outer iterations $\tau_{out}$ to $150$ while ensuring convergence, and we select the graph (or the \GtoG weights) with the highest validation accuracy.
We set $\lambda=1$ in training when considering the Laplacian regularization, as we expect the bilevel algorithm to learn this parameter by scaling the learned adjacency matrix.
When applying the Laplacian regularization fed with $\V{A}_{obs}$ on Cora, we set $\lambda=0.1$ after a grid search.
$\gamma$ in the graph regularization term is set to $1$ following a grid search on the set $\{10^{-3}, 10^{-2}, \ldots, 10\}$.

\subsection{Hypergradient scarcity with \GCN classifiers}
In this experiment, we consider a $2$-layer \GCN classifier in the bilevel framework.
We solve the edge refinement task 
\eqref{eq:bilevel_problem_learn_A} on the cheaters dataset, where $\ell$ in \cref{eq:inner-problem-laplace,eq:bilevel_problem_learn_A} is the CCE function.
\cref{fig:gcn_scarcity_bilevel}(left) depicts the initialization of the adjacency matrix.
It also shows what edges are to be optimized, that is, edges whose initialization is nonzero.
In \cref{fig:gcn_scarcity_bilevel}(right), we show the hypergradient at the outer iteration $9$, which is arbitrarily chosen, where it is clear that edges between unlabelled nodes far from the ones in the union $V_{out}\cup V_{tr}$ get no supervision during the training process.
Recall that $V_{tr} = \{0,1,\ldots, 32\} \cup  \{224,\ldots, 255\}$ and  $V_{out} = \{96,\dots, 160\}$.
This aligns with our findings, which state that edges between nodes at least $2$-hop from nodes in $V_{out}\cup V_{tr}$ receive zero hypergradients.
This, as seen in \cref{fig:propsed_solutions}, leads to a learned graph that overfits training nodes and even generalizes worse than $\V{A}_{obs}$.

\subsection{Hypergradient scarcity with the Laplacian regularization}\label{sec:exp_synthetic1}
We here examine hypergradient scarcity when adopting the Laplacian regularization in the inner problem.
We run the bilevel optimizer to solve the edge refinement task on the synthetic dataset 1.
The dataset corresponds to a regression problem, so $\ell$ in \cref{eq:inner-problem-laplace,eq:bilevel_problem_learn_A} is the MSE loss function.

In \cref{fig:Laplace_scarcity_bilevel}(bottom), we plot the absolute value of hypergradients at the outer iteration $6$ as a function of the edge cumulative distance to $V_{tr}$ and $V_{out}$, which is defined as follows: we compute $q+k$, the sum of distances to $V_{tr}$ and $V_{out}$, respectively, for its both endpoint nodes, then we take the minimum of the two results.
One observes the hypergradient scarcity phenomenon, since hypergradients decay exponentially as the edge distance increases.
This validates our analysis articulated in \cref{theorem:laplacian}.
In addition, we observe in practice that $\mu$ is nevertheless quite small, and that our bound in \cref{theorem:laplacian} is quite loose.
Another observation is that the decrease rate is higher when $V_{tr}$ is well-spread in the graph.
Deriving a tighter bound on the magnitude of hypergradients and investigating the link between the distribution of labelled nodes and this bound will be the subject of a future work.

\subsection{Testing solutions to mitigate hypergradient scarcity}
We run our experiments on the cheaters dataset using the $2$-layer \GCN as a classifier.
In each experiment, we run our bilevel optimization framework  with one of the suggested fixes.
We consider two criteria to measure the efficiency of each solution, the first one is counting the number of refined edges.
At any outer iteration, we say that an edge is refined if its learned weight is greater that one percent of the maximum learned edge weight at the same iteration.
Recall that we initialize the graph/GtoG with small weights ($\approx 10^{-5}$).
The second criterion is the validation accuracy.
The first criterion assesses the ability to alleviate hypergradient scarcity, while the second assesses the generalization to unseen nodes during training, and thus if the learned graph is meaningful.

\cref{fig:propsed_solutions} shows that all three fixes produce better results \wrt the first criterion, as the number of refined edges is larger at almost every iteration, with optimizing edges in $\V{A}_{obs}^6$ being the most efficient, and the \GtoG model and graph regularization having a similar performance.
Moreover, one notices that this number decreases with the iteration when refining edges (without fix) in $\V{A}_{obs}^6$, which is expected as only a small portion of edges receive supervision; however this portion is larger when refining $\V{A}_{obs}^6$.

Regarding the second criterion,  the \GtoG model and the graph regularization generalize well, as both combat hypergradient scarcity without increasing (or even by decreasing) the number of parameters to learn.
On the other hand, optimizing edges in $\V{A}_{obs}^6$ deteriorates performance in the test phase.
A likely explanation is that by expanding the graph, we increase the number of parameters to learn, which means a more complex model that is more likely to overfit training nodes.
This experiment illustrates that
\textbf{hypergradient scarcity is not the traditional overfitting} related to data/label scarcity, and resolving it does not necessarily promote better generalization.
\begin{figure}
	\centering
   \includegraphics[width=.24\textwidth]{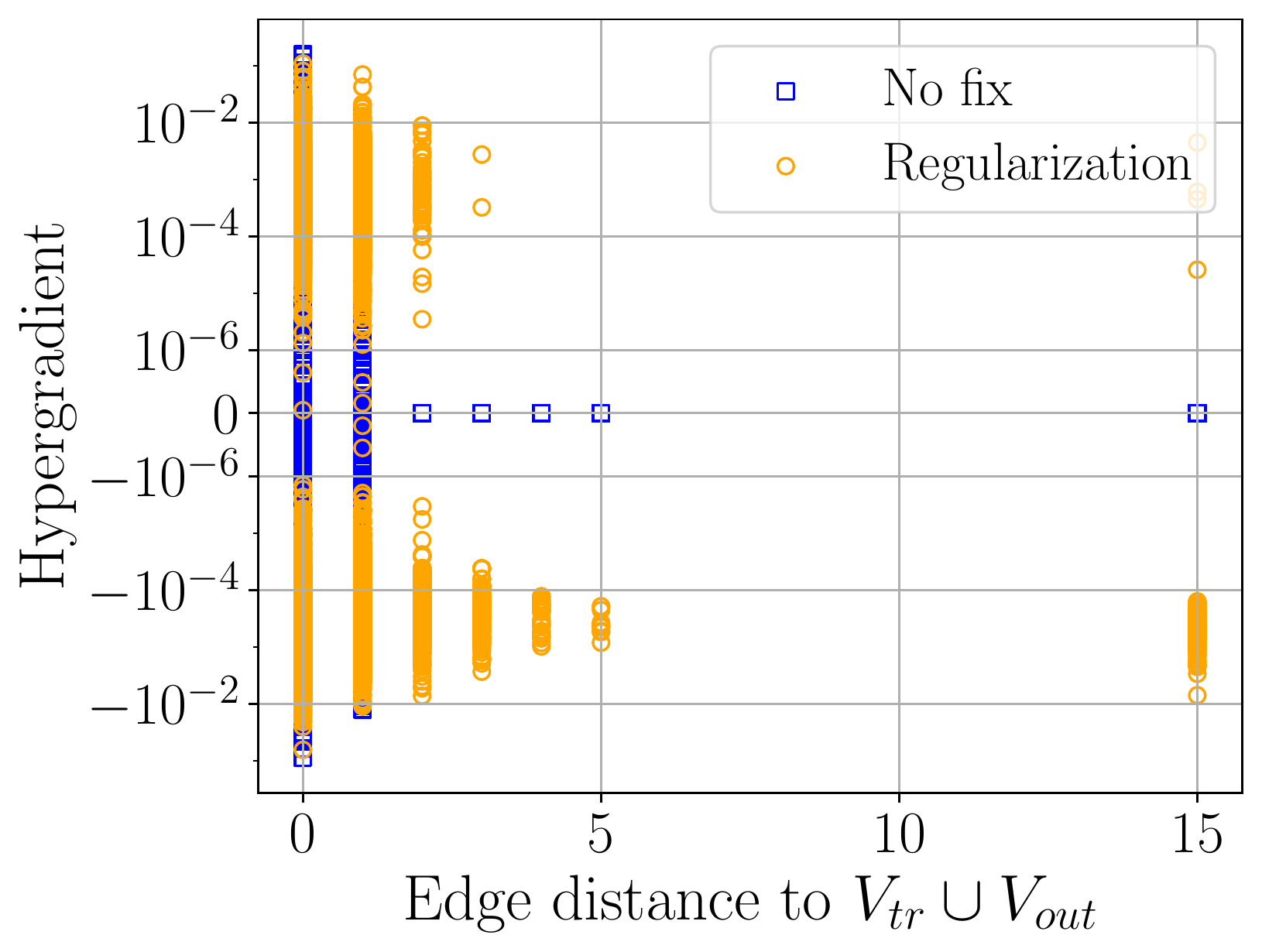}
   \includegraphics[width=.24\textwidth]{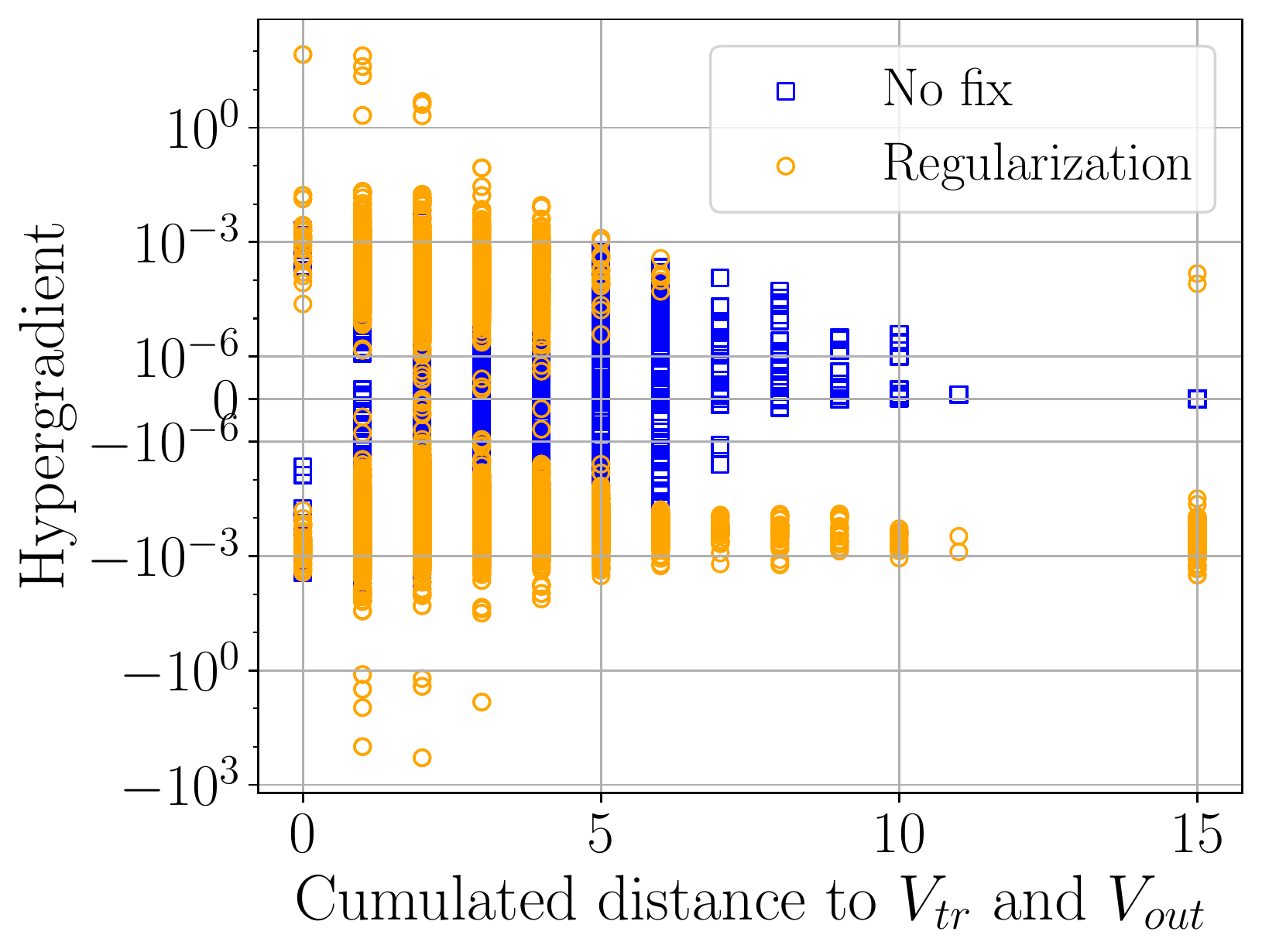}
	  \caption{Observing hypergradient scarcity and the effect of graph regularization on Cora.
      Left: adopting the \GCN classifier. Right: adopting the Laplacian regularization model.
   We plot the hypergradient against edge distance.
   In connected components without at least a node from each of $V_{tr}$ and $V_{out}$ in the Laplacian regularization case (or without a node from $V_{tr} \cup V_{out}$ in the \GCN case), edge distance is not defined. We assign the distance $15$ to edges in such components for visualization purpose.}
	  \label{fig:Cora_gradient_scarcity}
\end{figure}
\begin{table}
  \caption{Accuracies obtained on Cora when the classifier is trained using the output graph of the Bilevel Optimization (BO) framework, the same framework equipped with graph regularization, the same framework optimizing a \GtoG model. We also benchmark against GAM (the result is reported from the according paper) and against $\V{A}_{obs}$.
  For each classifier, we report test accuracy in the according first line and training accuracy on $V_{out}$ in the second one. Training accuracy on $V_{tr}$ equals $100\%$ for all methods.}
  \label{table:Cora_accuracy}
  \centering
  \begin{tabular}{lccccc}
    \toprule
     Graph & $\V{A}_{obs}$& BO& BO+regularization& BO+\GtoG& GAM   \\
    \midrule
    \multirow{2}{*}{\GCN} &$77.0$& $76.2$&$80.3$&$82.0$&$84.8$ \\
    &$77.4$& $94.9$&$94.1$&$97.4$&-\\
    \cmidrule(r){2-6}
    \multirow{2}{*}{Laplacian} &$71.7$& $76.2$&$78.3$&$76.2$&-\\
     &$71.0$& $81.9$&$83.2$&$83.5$&-\\
    \bottomrule
  \end{tabular}
\end{table}

\subsection{Results on Cora}
We use bilevel optimization \eqref{eq:bilevel_problem_learn_A} to solve an edge refinement task on Cora, trying both the \GCN and the Laplacian models.
Here the downstream task is a multi-label classification problem and $\ell$ is the CCE function.
We depict in \cref{fig:Cora_gradient_scarcity} the received hypergradient on edges at outer iteration $9$ as a function of their distance to labelled nodes.
For the Laplacian regularization case, that is the edge cumulative distance to $V_{tr}$ and $V_{out}$ as defined in \cref{sec:exp_synthetic1}.
Whereas to compute the edge distance in the \GCN case, we compute for each of its endpoint nodes its distance to $V_{tr} \cup V_{out}$, then we take the minimum.
In accordance with our analysis, the figure displays a null hypergradient for distances greater than $2$ in the \GCN case, while the Laplacian regularization scenario exhibits a hypergradient that diminishes exponentially with distance.

Regarding the generalization capacity, \cref{table:Cora_accuracy} shows that the learned graph is inferior to $\V{A}_{obs}$ in the \GCN case for the test error.
Given that the learned graph achieves $100\%, 94.9\%$ accuracies on $V_{tr}, V_{out}$, respectively, one concludes that hypergradient scarcity provokes overfitting.
This is indeed expected due to the extreme scarcity in the \GCN scenario, as edges of distance greater than $2$ keep their random initialization after the training process.
This is, however, not the case in the Laplacian regularization scenario as most edges are of distance less than $11$, thereby they do not exhibit damped hypergradients and the impact on generalization is not observed.

Next, we test the efficiency  of the proposed solutions to  mitigate hypergradient scarcity.
We do not try learning a power of $\V{A}_{obs}$ as the memory requirement goes beyond the limits we have access to.
Results in \cref{fig:Cora_gradient_scarcity} prove the efficiency of graph regularization as all edges receive non-zero hypergradients with a comparable magnitude to those on edges of small distance.
Note that hypergradients are received on the \GtoG weights when it is deployed, not on edges, so we do not depict them in this figure.
Regarding the impact on generalization, \cref{table:Cora_accuracy} shows that both fixes yield significant improvements in test accuracy over $\V{A}_{obs}$ with the \GCN classifier.
In the Laplacian regularization case, graph regularization produces a higher test accuracy, unlike the \GtoG model which generalizes equally good as when learning directly edge weights.
We also notice that \GCN model leads to superior results in all scenarios, with a notable gap for the \GtoG model and graph regularization, and when directly using $\V{A}_{obs}$.
This is expected, as the Laplacian regularization promotes similarity between connected nodes but, unlike \GCNs, is not a supervised-based method.
We finally point out that the bilevel optimization framework with either fix does not achieve state-of-the-art results produced by GAM with the same \GCN classifier.

Other experiments suggest that although \GtoG models alleviate  hypergradient scarcity, regardless of the number of neurons in its  layers, the generalization performance is sensitive to this number and if set large, clear overfitting is observed.

\section{Conclusion}
We studied hypergradient scarcity 
when deploying bilevel optimization in edge refinement tasks  
under the SSL settings.
This phenomenon consists in edges far from labelled nodes receiving scarce hypergradients when optimizing the graph and the classifier to improve classification performance.
We proved that this problem occurs for \GCN.
Replacing \GCNs by the Laplacian regularization model, does not resolve the issue; however, the phenomenon is less severe: we bounded the magnitude of hypergradients and proved they are exponentially damped with distance to labelled nodes.
To alleviate hypergradient scarcity, we proposed to resort to latent graph learning, graph regularization, and refining edges in a power of the observed adjacency matrix.
Our experiments validated our findings, and privileged the first two solutions over the latter.
Moreover, we show that alleviating the hypergradient scarcity does not necessarily alleviate overfitting.  

\bibliographystyle{IEEEtran}

\vfill

\end{document}